\documentclass[sigconf, manuscript]{acmart}

\usepackage[utf8]{inputenc}
\usepackage{mathtools,braket}
\usepackage{enumerate}
\usepackage{multirow, multicol}
\usepackage[plain]{algorithm}
\usepackage{algpseudocode}
\usepackage{subcaption}
\usepackage{varwidth}
\usepackage{url}
\usepackage{natbib}
\usepackage{soul} %for ul command 
\usepackage[ruled,linesnumbered,algo2e]{algorithm2e}
\usepackage{bm}
\usepackage{microtype}
\hypersetup{
	colorlinks = true, %Colours links instead of ugly boxes
	urlcolor = blue, %Colour for external hyperlinks
	linkcolor = blue, %Colour of internal links
	citecolor = blue %Colour of citations
}
\usepackage{amsmath,amsthm,amsfonts,latexsym,bbm,xspace,graphicx,float,mathtools,mathdots}
\usepackage{braket,caption,subcaption,ellipsis,xcolor,textcomp}
\usepackage{combelow} %for you, Rusins Freivalds
\newtheorem{theorem}{Theorem}[section]

\newtheorem{lemma}[theorem]{Lemma}

\newtheorem{proposition}[theorem]{Proposition}

\newtheorem{problem}[theorem]{Problem}

\theoremstyle{definition}
\newtheorem{definition}[theorem]{Definition}
\newtheorem{remark}[theorem]{Remark}

\newtheorem{observation}[theorem]{Observation}
\newtheorem{example}[theorem]{Example}

\newtheorem{thm}{Theorem}

\newcommand{\domain}{{\mathcal{X}}}
\newcommand{\group}{{\mathcal{Z}}}
\newcommand{\labelset}{{\mathcal{Y}}}
\newcommand{\advantaged}{{A}}
\newcommand{\disadvantaged}{{D}}
\newcommand{\distribution}{{\mathcal{P}}}
\newcommand{\score}{{\mathcal{S}}}
\newcommand{\cell}{{\mathcal{C}}}
\newcommand{\classifier}{f}
\newcommand{\prob}{\operatorname{Pr}}
\newcommand{\ex}{\mathbb E}
\newcommand{\repspace}{\mathcal{\Tilde{X}}}
\newcommand{\repfunc}{R}
\newcommand{\candidaterates}{\mathcal{I}}

\newcommand{\fFP}{f_r}

\newcommand{\fOpt}{f^{*}}

\newcommand{\RDP}{R_{DP}}
\newcommand{\RFP}{R_{PE}}

\newcommand{\loss}{\mathcal{L}}
\newcommand{\cost}{\alpha}
\newcommand{\positivityrate}{\pi}
\newcommand{\threshold}{\mathcal{T}}
\newcommand{\hypoclass}{\Tilde{\mathcal{H}}}
\newcommand{\hypoclassrep}{\Tilde{\mathcal{H}_R}}
\newcommand{\hypoclassdet}{\mathcal{H}}
\newcommand{\thresholdmass}{\Tilde{\mathcal{T}}}
\newcommand{\classifierDP}{f_r}
\newcommand{\accepted}{\mathcal{A}}
\newcommand{\rejected}{\mathcal{R}}
\newcommand{\costFR}{\text{CFR}}
\newcommand{\repfuncset}{R_{\text{fair}}}

\DeclareMathOperator*{\argmin}{argmin}

\AtBeginDocument{%
  \providecommand\BibTeX{{%
    \normalfont B\kern-0.5em{\scshape i\kern-0.25em b}\kern-0.8em\TeX}}}

\copyrightyear{2022} 
\acmYear{2022} 
\setcopyright{acmlicensed}\acmConference[FAccT '22]{2022 ACM Conference on Fairness, Accountability, and Transparency}{June 21--24, 2022}{Seoul, Republic of Korea}
\acmBooktitle{2022 ACM Conference on Fairness, Accountability, and Transparency (FAccT '22), June 21--24, 2022, Seoul, Republic of Korea}
\acmPrice{15.00}
\acmDOI{10.1145/3531146.3533209}
\acmISBN{978-1-4503-9352-2/22/06}

\begin{document}

\title{On the Power of Randomization in Fair Classification and Representation}
% \title{Optimal Randomized Fair Classification and Representation}


\author{Sushant Agarwal}
% \authornote{Work done during internship at Microsoft Research India.}
\email{agarwal.sus@northeastern.edu}
\affiliation{%
  \institution{Northeastern University}
  \country{USA}
}

\author{Amit Deshpande}
\email{amitdesh@microsoft.com}
\affiliation{%
  \institution{Microsoft Research}
  \country{India}
  }

\renewcommand{\shortauthors}{Sushant Agarwal and Amit Deshpande}

\begin{abstract}
Fair classification and fair representation learning are two  important problems in supervised and unsupervised fair machine learning, respectively. Fair classification asks for a classifier that maximizes accuracy on a given data distribution subject to fairness constraints. Fair representation maps a given data distribution over the original feature space to a distribution over a new representation space such that all classifiers over the representation satisfy fairness. In this paper, we examine the power of randomization in both these problems to minimize the loss of accuracy that results when we impose fairness constraints. Previous work on fair classification has characterized the optimal fair classifiers on a given data distribution that maximize accuracy subject to fairness constraints, e.g., Demographic Parity (DP), Equal Opportunity (EO), and Predictive Equality (PE). We refine these characterizations to demonstrate when the optimal randomized fair classifiers can surpass their deterministic counterparts in accuracy. We also show how the optimal randomized fair classifier that we characterize can be obtained as a solution to a convex optimization problem. Recent work has provided techniques to construct fair representations for a given data distribution such that any classifier over this representation satisfies DP. However, the classifiers on these fair representations either come with no or weak accuracy guarantees when compared to the optimal fair classifier on the original data distribution. Extending our ideas for randomized fair classification, we improve on these works, and construct DP-fair, EO-fair, and PE-fair representations that have provably optimal accuracy and suffer no accuracy loss compared to the optimal DP-fair, EO-fair, and PE-fair classifiers respectively on the original data distribution.

\end{abstract}




\maketitle

\section{Introduction}
% fairness of ML models
With the proliferation of machine learning models in several domains such as banking, education, healthcare, law enforcement, customer service etc. with direct social and economic impact on individuals, it has become imperative to build these models ethically and responsibly to avoid amplification of their biases. Assessing and mitigating bias to sensitive or underprivileged demographic groups is an important aspect of building ethical and responsible machine learning models, and understandably, a large part of recent literature on this has focused on the fairness of machine learning models \cite{barocas2016big,barocas2019fairmlbook,chouldechova2020snapshot}.
 
Fairness in classification has been an important topic of study for machine learning models because it corresponds to automated decision-making scenarios in sensitive application domains. Fairness of classification in literature has been largely divided into two subtopics, focusing on either individual fairness \cite{dwork2012fairness,bechavod2020metric,fleisher2021what,petersen2021post}, or group fairness \cite{hardt2016equality,zafar2017fairnessbeyond}. Group-fair classification requires a classifier to achieve certain outcome metrics to be equal or near-equal across sensitive demographic groups (e.g., race, gender). Legal precedents on disparate impact and the four-fifths rule \cite{barocas2016big} have given rise to statistical definitions and quantitative metrics of group-fairness such as Demographic Parity (DP), Equal Opportunity (EO), and Predictive Equality (PE) \cite{hardt2016equality,zafar2017fairnessbeyond}. Literature on group-fair classification has looked at various trade-offs of different metrics of group-fairness with each other \cite{kleinberg2016inherent} as well as their trade-offs with desiderata such as accuracy, interpretability, and privacy \cite{agarwal2020trade,agarwal2021trade}. Given the inevitable trade-off between accuracy and fairness, it is natural to ask the theoretical question of finding the \emph{optimal fair classifier} on a given data distribution that maximizes accuracy subject to certain group-fairness constraints. Previous work on this problem has tried to characterize optimal group-fair classifiers and proposed methods to train such group-fair classifiers \cite{pmlr-v81-menon18a,celis2020classification,DBLP:conf/nips/ChzhenDHOP19,corbettdavies2018measure,hardt2016equality}, and proposed methods to train group-fair classifiers \cite{kamiran2012data, agarwal2018reductions, zafar2017fairnessconstraints}. The characterization has considered both group-aware and group-blind classifiers, based on whether the use of sensitive attributes or group membership information at inference time is allowed or prohibited.

Classification by machine learning models is a part of supervised learning. Machine learning models that perform unsupervised learning (e.g., representation learning, generative models, clustering) do not directly correspond to decision-making scenarios but they are equally important in modern machine learning suites and equally vulnerable to the hidden risks of amplifying social and economic harms \cite{papakyriakopoulos2020bias,choi2020fair}. The outcome-based definitions of group-fairness such as DP and EO do not immediately apply to unsupervised learning models for fair representation and fair generation. Previous work on building fair representations has tried to incorporate certain differentiable proxies for fairness and accuracy objectives to train fair representations \cite{creager2019flexibly,ferber2021differentiable,tan2020learning}. 

This approach does not immediately guarantee group-fairness of models trained on the representations. To circumvent this, it has been natural to define the objective of fair representation as finding an embedding or a representation map from the original data to a new representation, so that any downstream classifier on the new representation must satisfy group-fairness constraints. Subsequently, it reduces our objective to only finding the classifier of maximum accuracy on the above fair representation. Another practical motivation for the above definition of fair representation is that it supports the scenario where a data regulator (e.g., a government body) releases the data to data users (e.g., a recruitment agency) in such a way that any model or classifier built by the data users on this data cannot violate a desirable fairness guarantee.

The fairness-accuracy trade-off in fair classification implies that there is an accuracy loss (or price of fairness) when we maximize accuracy subject to fairness constraints, as opposed to unconstrained accuracy maximization. In fair representation, accuracy-maximizing classifiers built on fair representations suffer an accuracy loss compared to the optimal fair classifiers on the original data distribution because there is a potential information loss when we goes from the original data distribution to its fair representation. The focus of our paper is to investigate how randomization can help minimize the accuracy loss in the fair classification and fair representation problems stated above. We study these problems for three of the most popular group-fairness constraints (DP, EO, PE), in the group-aware setting where the classifier and the representation map are allowed to use the sensitive group membership.
\paragraph{Roadmap of the paper}
The paper is organized as follows.
Section \ref{sec:results} briefly summarises our main results. Then,
Section \ref{sec:RW} discusses related works in the literature on fair classification and fair representation. 
We introduce the setup in Section \ref{sec:PF}, and formally state the problems we study in this paper. Section \ref{classification} characterizes the optimal randomized DP-fair, EO-fair, and PE-fair classifiers, and shows how they can be obtained as the solution to a convex optimization problem. Section \ref{rep} constructively characterizes the optimal randomized DP-fair, EO-fair, and PE-fair representations, that have provably optimal accuracy, and suffer no accuracy loss compared to the optimal fair randomized classifier on the original data distribution. We conclude and talk about potential directions for future work in
Section \ref{conclusion}.

\subsection{Our Results}\label{sec:results}
Our main contributions are summarized below.
\begin{itemize}
    \item We mathematically characterize the \emph{optimal randomized fair classifier} that maximizes accuracy  on a given data distribution subject to group-fairness constraints (DP, EO, PE). We do this in the group-aware setting and prove that the optimal randomized group-fair classifier is a randomized threshold classifier that can be obtained as the solution to a convex optimization problem. We give an example of a data distribution to show that the accuracy of the optimal randomized group-fair classifier can surpass the accuracy of the optimal deterministic group-fair classifier. This is in contrast with the deterministic Bayes Optimal classifier that maximizes accuracy on a given data distribution without any fairness constraints, and is optimal among all deterministic as well as randomized classifiers on the distribution.
    \item Extending our ideas for randomized fair classification, we construct randomized fair representations for DP, EO, and PE constraints such that the optimal classifier on this representation incurs no \emph{cost of fair representation}, or in other words, no accuracy loss compared to the optimal fair randomized classifier on the original data distribution. This improves upon previous work on fair representations that primarily focuses on DP and gives no or weak provable guarantees for accuracy.
\end{itemize}

\subsection{Related Work}\label{sec:RW}

Given a joint data distribution on features and binary class labels, it is well-known that the optimal classifier is given by the Bayes Optimal classifier that applies a threshold of $1/2$ on the class-probability function, i.e., the probability of the class label being positive given the feature vector. A natural extension of this under fairness constraints is to characterize the optimal fair classifier on a given joint data distribution over features, class labels and group membership, that maximizes accuracy subject to group-fairness constraints. This question can be asked for both group-aware and group-blind classifiers, based on whether using a sensitive attribute is allowed or prohibited for inference. Previous work has studied this extensively: Given a predictive score function, Hardt et al.~\cite{hardt2016equality} characterize the optimal group-aware fair classifier for EO as a group-dependent threshold classifier on the predictive score value. Under the assumption that the class-probability function has a strictly positive density on $[0, 1]$, Corbett-Davies et al.~\cite{corbettcost} characterize the optimal group-aware fair classifiers for DP and PE as group-dependent threshold classifiers on the class-probability function. Menon and Williamson~\cite{pmlr-v81-menon18a} characterize the optimal group-blind (and group-aware) fair classifiers for DP and EO as instance-dependent (and group-dependent) threshold classifiers on the class-probability function. Celis et al.~\cite{celis2020classification} extend this characterization to optimal fair classifiers under multiple simultaneous fairness constraints. Chzhen et al.~(Proposition $2.3$ in \cite{DBLP:conf/nips/ChzhenDHOP19}) give an explicit expression for the optimal deterministic group-aware fair classifier for EO, describing it as a recalibrated Bayes classifier. Zhao et al.~\cite{zhao2019inherent} characterize the optimal group-aware fair classifier for DP using oracle access to the Bayes optimal classifiers on the underlying sensitive groups. Agarwal et al.~\cite{agarwal2018reductions} reduce fair classification to a sequence of cost-sensitive classifications, whose solutions give a randomized classifier of least empirical error subject to the desired fairness constraints. Their randomized classifier samples a random classifier from a given hypothesis class and then uses it to make the prediction. Our work refines the above results on the characterization of optimal fair classifiers to demonstrate when randomized fair classifiers have an advantage over deterministic fair classifiers. We further extend these ideas for optimal randomized fair classification to construct randomized fair representations that achieve optimal accuracy and suffer no accuracy loss compared to the optimal fair classifiers on the original data distribution. In recent concurrent work, Zeng et al.~\cite{https://doi.org/10.48550/arxiv.2202.09724} characterize the optimal fair randomized classifiers for DP, EO, and PE, however, they do not highlight the advantages of randomization, and do not consider the problem of fair representations. We note that our characterization results assume full access to the data distribution. In practice, this does not hold, and one only has access to a finite sample drawn from the distribution. To (approximately) realise our classifiers in practice, we can make use of techniques similar to the ones suggested in many of the above works, but we do not focus on this.

The outcome-based notions of fairness for classification need to be modified to apply to unsupervised learning problems such as representation learning. 

There is a long line of work on defining and constructing fair representations~ \cite{delbarrio2018obtaining,DBLP:conf/nips/ChzhenDHOP20,Silvia_Ray_Tom_Aldo_Heinrich_John_2020,feldman2015certifying, pmlr-v28-zemel13,mcnamara201costs,madras2018learning, ruoss2020learning,louizos2017variational, johndrow2017algorithm}. Feldman et al.~\cite{feldman2015certifying} construct DP-fair representations by introducing the paradigm of total repair, which involves transforming the original training dataset to a new dataset, such that the protected attribute is now independent of the class label. In particular, total repair involves mapping the group-wise conditional distributions from the original feature space to a common distribution in the new representation space, so that it is impossible to predict the protected attribute from the new sanitised dataset. This transformation is done in a manner that preserves as much of the information in the original dataset as possible. Many follow up works also give techniques for DP-fair representation that fall under the paradigm of total repair \cite{delbarrio2018obtaining, johndrow2017algorithm}. However, most of these works provide weak or no accuracy guarantees for the optimal classifier on the fair representation. Our construction for DP-fair representation also falls under the paradigm of total repair, but in contrast, we construct a fair representation such that the optimal classifier on this representation incurs no accuracy loss compared to the optimal fair classifier on the original data distribution. In addition, we also construct fair representations for notions of fairness apart from DP, such as EO and PE. Our representation technique works for any distribution over an arbitrary feature space, which does not always hold for previous techniques that use geometric properties of continuous feature spaces. McNamara et al.~\cite{mcnamara201costs} and Zhao et al.~\cite{zhao2019inherent} also study the cost of using a fair representation in terms of accuracy loss. McNamara et al.~\cite{mcnamara201costs} term this loss as the \emph{cost of mistrust}, which is similar to what we call the \emph{cost of fair representation}. We define the cost of fair representation as the difference between the accuracy of the optimal fair classifier on the original data distribution and the accuracy of the optimal classifier on the fair representation. McNamara et al.~\cite{mcnamara201costs} provide a closed form expression for the cost of mistrust, and also derive an upper bound on its value. Their result works for a fairness-regularized accuracy maximization objective that does not imply a bound on the cost of fair representation defined in our paper. 

Apart from the above results that are closest to our paper, there is also plenty of work on fair representation learning in practice \cite{pmlr-v28-zemel13, louizos2017variational, madras2018learning, creager2019flexibly}. Zemel et al.~\cite{pmlr-v28-zemel13} develop a heuristic technique for learning DP-fair representations, where they formulate fairness as an optimization problem of finding an intermediate data representation that best encodes the data, while removing any information about the sensitive group memberships. Their optimization problem has terms to encourage both these goals simultaneously. However, Zemel et al.~\cite{pmlr-v28-zemel13} and its follow-up works~ \cite{louizos2017variational, madras2018learning,creager2019flexibly} provide heuristics with no provable accuracy or fairness guarantees.

While we focus on the benefits of randomization, it is important to note that randomized decisions for fair machine learning can have practical limitations in real-world deployment, e.g., inconsistent decisions for the same individual when repeated, as discussed in Cotter et al.~\cite{NEURIPS2019_5fc34ed3}.

\section{Problem Formulation}\label{sec:PF}
\subsection{Setup}
Consider a learning problem, where we are given a distribution $\distribution$ over $\domain \times \group \times \labelset$, where $\group = \{\advantaged, \disadvantaged \}$ represents the protected group membership, $\domain$ represents all the other features,
and $\labelset = \{0,1\}$ represents the label set (we adopt the standard convention of associating the label $1$ with success or acceptance). Our results also hold when there are multiple groups, but for ease of exposition, we restrict our analysis to the case of 2 groups (advantaged group $A$, and disadvantaged group $D$). A deterministic group-aware classification rule $f$ is a map
\begin{equation*}
    \classifier: \domain \times \group \rightarrow \{0,1\},
\end{equation*}
that assigns a binary label to each point (also known as feature vector) in the feature space $\domain \times \group$. We denote the hypothesis class of all such deterministic functions $f$ by $\hypoclassdet$.

The aim of the learning problem is to find a classifier in the hypothesis class that minimises a loss function. We consider the standard 0-1 loss function $\ell$, whose expected value is given by
\begin{equation}\label{loss}
  \loss(f) =  \ex[\ell(f)] =  \prob[f(X,Z)\neq Y],
\end{equation}
where the probability is over $(X,Z,Y)\sim\distribution$.
\begin{remark}
Henceforth, all probabilities will be over $(X,Z,Y)\sim\distribution$, unless explicitly stated.
\end{remark}
Our results also work for the more general loss function $\ell_\cost$, known in literature as cost-sensitive risk \cite{pmlr-v81-menon18a}, that assigns a weight $\cost$ to False Positive errors, and a weight $(1-\cost)$ to False Negative errors. The expected loss is a weighted linear combination of the False Positive rate (FPR) and False Negative rate (FNR), where
\begin{align*}
\text{FPR}(\classifier) & = \prob[f(X,Z)=1 ~|~ Y = 0] \quad \text{and} \\
\text{FNR}(\classifier) & = \prob[f(X,Z)=0 ~|~ Y = 1].
\end{align*}
We want to minimize the expected loss, given by
\begin{equation*}
    \loss_\cost(\classifier) = \ex[\ell_\cost(f)] = \cost(1-\positivityrate) \text{FPR}(\classifier) + (1-\cost)\positivityrate \text{FNR}(\classifier),
\end{equation*}
where $\positivityrate = \prob[Y=1]$. By assigning equal weight to False Positive errors and False Negative errors by setting $\cost = \frac{1}{2}$, we recover the standard 0-1 loss function, given by Equation \ref{loss}.
\begin{remark}
If $\cost$ is not mentioned, it is assumed to be $\frac{1}{2}$. Our results hold for any $0 < \cost < 1$, but for ease of exposition, we often restrict our analysis to $\cost = \frac{1}{2}$.
\end{remark}

In the standard 0-1 loss setting, the optimal classifier $\fOpt$

(the Bayes Optimal Classifier) is given by
\begin{equation*}
    \fOpt(x,z)= \threshold_{\frac{1}{2}}(\prob[Y=1 ~|~ X=x, Z=z]),
\end{equation*}
and $\threshold_\gamma(\beta)$ is the threshold function that outputs 1 if $\beta \geq \gamma$, and 0 otherwise. The optimal classifier for cost-sensitive risk $\ell_\cost$ has a very similar form \cite{corbettdavies2018measure}, and is given by 
\begin{equation*}
    \fOpt_\alpha(x,z)= \threshold_\cost(\prob[Y=1 ~|~ X=x, Z=z]).
\end{equation*}
The only difference is that, instead of thresholding the success probability for a point at $\frac{1}{2}$, we threshold at $\cost$. 

In many situations we may want to consider a fairness-aware learning problem, that aims to find the most accurate hypothesis in $\hypoclassdet$ that is also group-fair, and does not discriminate based on the protected group feature $\group$. We consider a few popular notions of fairness, that we define below.
\begin{definition}[Demographic Parity]
$f$ satisfies Demographic Parity with selection rate $r$, if
\begin{equation*}
   \prob[f(X,Z)=1 ~|~ Z = A] = \prob[f(X,Z)=1 ~|~ Z = D] = r.
\end{equation*}
\end{definition}
\begin{definition}[Predictive Equality, or Equal FPR]
$f$ satisfies Predictive Equality if the false positive rates are equal for both groups, i.e.,
\begin{equation*}
    \prob[f(X,Z)=1 ~|~ Y=0, Z = A] = \prob[f(X,Z)=1 ~|~ Y=0, Z = D].
\end{equation*}
\end{definition}
\begin{definition}[Equal Opportunity, or Equal FNR \cite{hardt2016equality}]
$f$ satisfies Equal Opportunity if the false negative rates are equal for both groups, i.e.,
\begin{equation*}
    \prob[f(X,Z)=0 ~|~ Y=1, Z = A] = \prob[f(X,Z)=0 ~|~ Y=1, Z = D].
\end{equation*}
\end{definition}
\begin{remark}
Classifiers satisfying these notions of fairness will be referred to as DP-fair, PE-fair, and EO-fair respectively. Our results hold for all three fairness notions. However, we can obtain analogous results for EO from the same proof techniques as that of PE. Hence, we only discuss the results for PE. In addition, some previous works also consider equal True Positive rate (TPR), and equal True Negative rate (TNR), as notions of fairness. However, obtaining equal TPR is equivalent to obtaining equal FNR, and obtaining equal TNR is equivalent to obtaining equal FPR, and hence we do not separately discuss these two notions. 
\end{remark}

\subsection{Randomized Fair Classification Problem}
We now consider a relaxation of deterministic classifiers, and allow a classifier to be randomized. A randomized classification rule $f$ is a function
\begin{equation*}
    \classifier: \domain \times \group \rightarrow [0,1],
\end{equation*}
where $\classifier(x,z)$ denotes the probability of $(x,z) \in \domain \times \group$ being mapped to $1$. We denote the hypothesis class of all such randomized functions $f$ by $\hypoclass$. The notions of loss and fairness previously defined for the deterministic case can be easily modified for the randomized case. We investigate whether the ability to use randomization could lead to more powerful classifiers, and consequently study the following problem.
\begin{problem}\label{prob1}
What is the most accurate classifier in $\hypoclass$ satisfying a particular notion of fairness?
\end{problem}
Note that ability to use randomization can only lead to more powerful classifiers, because $\hypoclassdet \subsetneq \hypoclass$. In particular, a randomized classification rule can map a feature vector $(x,z)$ to 1 with probability $\classifier(x,z)$, and to 0 with probability $(1-\classifier(x,z))$. In contrast, a deterministic classifier can either choose to map $(x,z)$ strictly to 1, or strictly to 0, which is captured in the randomized setting by setting $\classifier(x,z)=1$, or $\classifier(x,z)=0$, respectively. To see how randomization can improve the accuracy of fair classifiers, let us look at an example.
\begin{example}\label{eg}
Consider the following distribution $\distribution$ over $\domain \times \group \times \labelset$, where $\domain = \{x_1, x_2\}$, $\group = \{\advantaged, \disadvantaged \}$, and $\labelset = \{0,1\}$.

\begin{equation*}
    \distribution(x_1, \advantaged, 1) = \frac{3}{8}, \distribution(x_1, \advantaged, 0) = \frac{1}{8}, \distribution(x_1, \disadvantaged, 1) = \frac{1}{8}, \distribution(x_1, \disadvantaged, 0) = \frac{1}{8}
\end{equation*}
\begin{equation*}
    \distribution(x_2, \advantaged, 1) = 0, \distribution(x_2, \advantaged, 0) = 0, \distribution(x_2, \disadvantaged, 1) = 0, \distribution(x_2, \disadvantaged, 0) = \frac{1}{4}.
\end{equation*}
There are only 2 deterministic classifiers satisfying DP, either the constant 1 classifier $f_1$, or the constant 0 classifier $f_0$, with $\loss(f_1) = \loss(f_0) = \frac{1}{2}$. On the other hand, consider the following randomized classifier $f$, where
\begin{equation*}
    f(x_1, \advantaged) = \frac{1}{2}, f(x_1, \disadvantaged) = 1, f(x_2, \advantaged) = f(x_2, \disadvantaged) = 0.
\end{equation*} 
It is easy to see that $f$ satisfies DP, and $\loss(f) = \frac{3}{8}$, hence improving the accuracy of the optimal fair deterministic classifiers $f_0$ and $f_1$.
\end{example}

\subsection{Randomized Fair Representation Problem}
A common way used to obtain fair classifiers is to pre-process the data by mapping the data from the original domain space to a new representation space. Essentially, a data regulator modifies the data before the data user is allowed to train a model on it. The goal of the data regulator is to sanitise the data such that models trained on it are fair, while also maintaining predictive accuracy by preserving as much information of the original data set as possible. A data representation is a randomized mapping
\begin{equation*}
    \repfunc: \domain \times \group \rightarrow \repspace,
\end{equation*}
for some representation space $\repspace$. A deterministic classification rule $g$ over this representation is a map $g: \repspace \rightarrow \{0,1\}$,
that assigns a binary label to each point in the new representation space $\repspace$. For a fixed $R$, we define the following hypothesis class $\hypoclassrep = \{g \circ R ~|~ g \text{ is a deterministic classifier over } R \}$.

A data regulator wants to ensure that the classifier trained by the data user on the sanitised data is fair. However, the data user could be adversarial, and actively try and be unfair. Hence, the data regulator needs to ensure that every classifier over the representation is fair, and this motivates the following definition of a fair representation. 
\begin{definition}[Fair Representation]
A representation $R$ is fair if every $h \in \hypoclassrep$ is fair.
\end{definition}
Note that for any $R, \hypoclassrep \subseteq \hypoclass$, and hence the accuracy of the optimal classifier in $\hypoclass$ cannot be lesser than the accuracy of the optimal classifier in $\hypoclassrep$. Similarly, if the representation $R$ is fair, then the accuracy of the optimal fair classifier in $\hypoclass$ cannot be lesser than the accuracy of the optimal classifier in $\hypoclassrep$. Hence, the fairness of the representation may come at a cost, and we define the following quantity below, that measures the difference in the accuracy of the optimal fair classifier in $\hypoclass$, and the optimal classifier in $\hypoclassrep$.

\begin{definition}[Cost of Fair Representation $R$ (denoted by $\costFR(R)$)]
\begin{equation*}
\costFR(R) = \min_{h \in \hypoclassrep}\loss(h) - \min_{f \in \hypoclass}\loss(f).
\end{equation*}
\end{definition}
Consequently, we focus on the following problem.
\begin{problem}\label{prob2}
Let $\repfuncset$ be the set of fair representations. What is the representation $R' \in \repfuncset$ that yields the minimal cost of fair representation? 

\end{problem}
\begin{remark}
 For simplicity of analysis in later proofs, we only consider deterministic classifiers over the randomized representations. It is easy to see that allowing for randomized classifiers over randomized representations provides no additional power, and does not affect the minimum possible cost of fair representation, as the randomness of the optimal classifier on a given representation could be incorporated in the representation itself.
\end{remark}
\section{Characterizing the Randomized Fair Classifier of Maximum Accuracy}\label{classification}
We now introduce some notation and terminology that we will need later on.
\begin{definition}[Cell \cite{kleinberg2019simplicity}]
Consider a randomized partition of the feature space $\domain \times \group$ into multiple disjoint components. We call these components cells, and denote a cell by $\cell$. 
\end{definition}
A natural partition we will encounter is the following, where each feature vector $(x,z) \in \domain \times \group$ is a separate cell, denoted by $\cell_{x,z}$. We consider randomized partitions, i.e., a feature vector $(x,z)$ does not have to strictly lie within one cell, but could be divided into multiple cells. For example, feature vector $(x,z)$ could be mapped to a cell $\cell_1$ with probability $\gamma(x,z)$, and to another cell $\cell_2$ with probability $(1-\gamma(x,z))$.
\begin{definition}[Score \cite{kleinberg2019simplicity}]
The score $\score$ of a cell $\cell$ is the probability that a random instance drawn from $\distribution$, given that it lies in $\cell$, has label $1$, i.e.,
\begin{equation*}
    \score(\cell) = \prob[Y=1~|~(X, Z) \in \cell].
\end{equation*}
\end{definition}

Previously, we encountered classifiers that threshold based on success probabilities (i.e., scores), such as the Bayes Optimal classifier. We formally define them below.
\begin{definition}[Score Threshold Classifiers]
Given $t \in [0,1]$, a score-threshold classifier $\threshold_t$ maps $\cell_{x,z}$ to 1 if $\score(\cell_{x,z}) \geq t$, and to 0 otherwise. 
\end{definition}

Randomized classifiers give us the ability to threshold by probability mass, instead of just thresholding by scores. To explain this better, we introduce the notion of group-wise sorted cells.
\begin{definition}[Group-wise Sorted Cells]
Define $\cell_z = \bigcup\limits_{x \in \domain} \cell_{x,z}$, where the component cells of $\cell_\advantaged$ and $\cell_\disadvantaged$ are arranged in descending order of scores $\score$. If two or more cells from the same group have the same score, merge them. This gives us cells sorted in strictly decreasing order of scores.
\end{definition}
By $\cell_z(t)$, denote the topmost cells of $\cell_z$ comprising of $t$ fraction of the total probability mass of $\cell_z$, i.e., $\frac{\distribution(\cell_z(t))}{\distribution(\cell_z)} = t.$
Note that this may involve splitting a cell into 2 parts randomly. For example, in Example \ref{eg}, $\cell_{\advantaged}(\frac{1}{2})$ would involve splitting $\cell_{x_1,\advantaged}$ into two equal parts randomly. However, in the deterministic setting, only $\cell_{\advantaged}(0)$ and $\cell_{\advantaged}(1)$ are defined, and $\cell_{\advantaged}(\frac{1}{2})$ does not make sense. Hence, in the deterministic setting, we cannot in general threshold by probability mass for an arbitrary threshold value. 

\begin{definition}[Mass Threshold Classifiers] Let the instances in the feature space accepted and rejected by $f$ be denoted by $\accepted(f)$ and $\rejected(f)$ respectively.
By $\thresholdmass_t$, we denote the mass-threshold classifier such that $\accepted(\thresholdmass_t) = \bigcup\limits_{z \in \group} \cell_z(t)$, and $\rejected(\thresholdmass_t) = \bigcup\limits_{z \in \group}(\cell_z - \cell_z(t))$.
\end{definition}
For example, in Example \ref{eg}, the randomised classifier $f$ we constructed is actually the randomized mass-threshold classifier $\thresholdmass_{\frac{1}{2}}$. Note that a mass-threshold classifier needs to use randomization on (at most) only one cell in each group, and is deterministic on all other cells.
\subsection{Characterizing the Optimal Randomized DP-Fair Classifier}

Here, we answer Problem \ref{prob1} for DP, with the following characterisation result.
\begin{thm}\label{thm:characDP}
There exists $r' \in [0,1]$, such that the most accurate DP-fair classifier in $\hypoclass$ is given by the mass-threshold classifier $\thresholdmass_{r'}$.
\end{thm}
\begin{proof}
We start by proving Proposition \ref{propDP}, and use it to prove Lemma \ref{lemmaDP}, which in turn implies Theorem \ref{thm:characDP}.
\begin{proposition}\label{propDP}
For any $r \in [0,1], \thresholdmass_{r}$ is DP-fair.
\end{proposition}
\begin{proof}
Let the elements of a cell $\cell$ accepted and rejected by $f$ be denoted by $\accepted(f(\cell))$ and $\rejected(f(\cell))$ respectively. For both $z=\advantaged$, and $z=\disadvantaged$, $\accepted(\thresholdmass_{r}(\cell_z))=\cell_z(r)$, and $\rejected(\thresholdmass_{r}(\cell_z)) = \cell_z - \cell_z(r)$.

Selection rate of $\thresholdmass_{r}$ on group $z$ is given by 
\begin{equation*}
    \frac{\distribution(\accepted(\thresholdmass_{r}(\cell_z)))}{\distribution(\cell_z)} = \frac{\distribution(\cell_z(r))}{\distribution(\cell_z)}.
\end{equation*}
And by definition of $\cell_z(r)$,
\begin{equation*}
    \frac{\distribution(\cell_z(r))}{\distribution(\cell_z)} = r.
\end{equation*}
Hence $\thresholdmass_{r}$ has the same selection rate (i.e., $r$) for both groups, and satisfies Demographic Parity.
\end{proof}
We then proceed to prove Lemma \ref{lemmaDP}.
\begin{lemma}\label{lemmaDP}
The most accurate DP-fair classifier in $\hypoclass$ with a fixed selection rate $r \in [0,1]$ is given by $\thresholdmass_r$.
\end{lemma}
\begin{proof}
Consider an arbitrary classifier with selection rate $r$, denoted by $\classifierDP$, such that $\classifierDP \neq \thresholdmass_r$. We will now argue that $\loss(\thresholdmass_r) < \loss(\classifierDP)$.

We know that the (expected) loss is the (weighted) sum of false positive and false negative errors. Hence,
\begin{equation*}
    \loss(f) = \distribution(\accepted(f)) (1- \score(\accepted(f))) +  \distribution(\rejected(f)) \score(\rejected(f)).
\end{equation*}
Since we know that 
\begin{equation*}
    \distribution(\accepted(f_r)) = \distribution(\accepted(\thresholdmass_r)) = r, \text{ and }
    \distribution(\rejected(f_r)) = \distribution(\rejected(\thresholdmass_r)) = 1 - r,
\end{equation*}
it suffices to show that
\begin{equation*}
    \score(\accepted(\thresholdmass_r)) > \score(\accepted(\classifierDP)), \text{ and }
    \score(\rejected(\thresholdmass_r)) < \score(\rejected(\classifierDP)).
\end{equation*}
We prove the former first. We can prove the latter with exactly the same reasoning, and hence omit the details. It suffices to show that
\begin{equation*}
    \score(\accepted(\thresholdmass_r(\cell_z))) > \score(\accepted(\classifierDP(\cell_z))), \text{ if }
    \accepted(\thresholdmass_r(\cell_z)) \neq \accepted(\classifierDP(\cell_z)).
\end{equation*}
To show the above, it is enough to observe that
\begin{equation*}
    \score(\accepted(\thresholdmass_r(\cell_z)) - (\accepted(\thresholdmass_r(\cell_z)) \cap \accepted(\classifierDP(\cell_z)))) >
\end{equation*}
\begin{equation*}
    \score(\accepted(\classifierDP(\cell_z)) - (\accepted(\thresholdmass_r(\cell_z)) \cap \accepted(\classifierDP(\cell_z)))).
    %\score(\accepted(\thresholdmass_r(\cell_z)) \cap \accepted(\classifierDP(\cell_z)))
\end{equation*}
This is easy to see, because 
\begin{equation*}
    (\accepted(\thresholdmass_r(\cell_z)) - (\accepted(\thresholdmass_r(\cell_z)) \cap \accepted(\classifierDP(\cell_z)))) \subseteq \cell_z(r), 
\end{equation*}
and 
\begin{equation*}
    (\accepted(\classifierDP(\cell_z)) - (\accepted(\thresholdmass_r(\cell_z)) \cap \accepted(\classifierDP(\cell_z)))) \subseteq (\cell_z - \cell_z(r)),
\end{equation*}
and given $P \subseteq \cell_z(r), Q \subseteq (\cell_z - \cell_z(r))$, such that $P \cap Q = \phi$, then $\score(P) > \score(Q)$, because the component cells of $\cell_z$ are sorted in strictly descending order of scores.

This completes the proof of the lemma.
\end{proof}
Armed with this lemma, it is easy to see that the most accurate classifier in $\hypoclass$ satisfying Demographic Parity is given by $\thresholdmass_{r'}$, where $r'$ denotes the optimum selection rate, i.e., \begin{equation*}
    r' \in \argmin_{r \in [0,1]}\loss(\thresholdmass_r),
\end{equation*}
which completes the proof.
\end{proof}
\subsubsection{Finding the Optimal Randomized DP-Fair Classifier by Convex Optimization}
Having characterized the optimal classifier, we now discuss how to find it as the solution to a convex optimization problem. The main result is the following theorem. 
\begin{thm}\label{thm:convexDP}
$\loss(\thresholdmass_r)$ is a convex, piecewise linear, continuous function over $[0,1]$. The optimal classifier $\thresholdmass_{r'}$ as in Theorem \ref{thm:characDP}, where
\begin{equation*}
    r' \in \argmin_{r \in [0,1]}\loss(\thresholdmass_r),
\end{equation*}
 is therefore the solution to a convex optimisation problem.
\end{thm}
We first introduce some terminology that will be needed in the upcoming proofs.
\begin{definition}[Score boundaries]
Consider the component cells of groupwise sorted $\cell_\advantaged$ and $\cell_\disadvantaged$. Then,
\begin{equation*}
    \candidaterates = \candidaterates_A \cup \candidaterates_D, \text{ where }
    \candidaterates_z = \{0\} \cup \{r ~|~ \exists \cell_z^i \in \cell_z ~|~ \cell_z^{\leq i} = \cell_z(r) \},
\end{equation*}

and $\cell_z^{\leq i}$ denotes the topmost cells of $\cell_z$ comprising of all cells until and including $\cell_z^i$. $\candidaterates_z$ essentially consists of all the boundary points in $\cell_z$, where the component cells change score.
\end{definition}
We now proceed with the proof.
\begin{proof}
We start by showing that $\loss(\thresholdmass_r)$ is piecewise linear and continuous. We do this by proving that the loss function is linear between any two points in the set $\candidaterates$.

Consider any $r_i \in \candidaterates = \{0, r_a, r_b,\ldots,1\}$ in sorted order, and define a corresponding cell $\cell_i$, where

\begin{equation*}
   \cell_i = \accepted(\thresholdmass_{r_i}) - \accepted(\thresholdmass_{r_{i-}}),
\end{equation*}
and $r_{i-}$ denotes the element in $\candidaterates$ preceding $r_{i}$.

Consider any $k \geq 0$, such that, $r_{i-} + k \leq r_{i}$. We know that
\begin{equation*}
    \score(\accepted(\thresholdmass_{r_{i-}+k}) - \accepted(\thresholdmass_{r_{i-}})) = \score(\cell_i),
\end{equation*}
because the cell $\cell_i$ is split into two parts at random, and hence both parts have the same score as $\cell_i$.

We know that loss is the (weighted) sum of false positive and false negative errors. Hence,
\begin{equation*}
    \loss(\thresholdmass_{(r_{i-}+k)}) = \loss(\thresholdmass_{r_{i-}}) + k (1-\score(\cell_i)) - k \score(\cell_i)
\end{equation*}
\begin{equation*}
    = \loss(\thresholdmass_{r_{i-}}) + k(1- 2\score(\cell_i)).
\end{equation*}
Looking at the form of the expression of the loss as above, it is easy to see that the loss function is linear in the interval $[r_{i-}, r_{i}]$, for any $r_i$, and is therefore piecewise linear and continuous over $[0,1]$.

We now proceed to prove that the loss function is convex. We begin by observing the following.
\begin{observation}\label{DPconvex}
Consider the interval between $[r_{i-}, r_{i}]$, for any $r_i \in \candidaterates$. The slope of the loss function $\loss(\thresholdmass_{r})$ at any point in this interval is $(1-2\score(\cell_i))$, and hence
\begin{enumerate}
    \item $\loss(\thresholdmass_{r})$ is strictly decreasing if $\score(\cell_i) > \frac{1}{2}$, and the decrease is steeper if the score is higher. 
    \item $\loss(\thresholdmass_{r})$ is strictly increasing if $\score(\cell_i) < \frac{1}{2}$, and the increase is steeper if the score is lower.
    \item $\loss(\thresholdmass_{r})$ is constant if $\score(\cell_i) = \frac{1}{2}$. 
\end{enumerate}
\end{observation}
Because the component cells of $\cell_z$ are sorted in strictly descending order of scores, it is easy to see that, for every $r_i$, $\score(\cell_{i}) < \score(\cell_{i-})$.
Hence, there exists a cell $\cell_{i'}$, such that for every $r_j, r_k \in \candidaterates$ satisfying $r_j < r_{i'} < r_k$, one of the following cases hold.
\begin{enumerate}
    \item $\score(\cell_{j}) > \frac{1}{2} > \score(\cell_{i'}) > \score(\cell_{k})$.
    \item $\score(\cell_{j}) > \score(\cell_{i'}) > \frac{1}{2} > \score(\cell_{k})$.
    \item $\score(\cell_{j}) > \score(\cell_{i'}) = \frac{1}{2} > \score(\cell_{k})$.
\end{enumerate}
In case (1), the loss function is strictly decreasing with shrinking rate of decrease uptil $r_{i'-}$, and strictly increasing with growing rate of increase after $r_{i'-}$. 
In case (2), the loss function is strictly decreasing with shrinking rate of decrease uptil $r_{i'}$, and strictly increasing with growing rate of increase after $r_{i'}$. 
In case (3), the loss function is strictly decreasing with shrinking rate of decrease uptil $r_{i'-}$, constant between $r_{i'-}$ and $r_{i'}$, and strictly increasing with growing rate of increase after $r_{i'}$. In each case, it is easy to see that the loss function is convex.

This concludes the proof.
\end{proof}
\begin{remark}
We note that the optima is unique if there exists no cell $\cell_i$, such that $\score(\cell_i) = \frac{1}{2}$.
\end{remark}
In addition, to find the optimum selection rate $r'$, we just need to evaluate $\loss(\thresholdmass_r)$ for all $r \in \candidaterates$, instead of for all $r \in [0,1]$, as a corollary of Observation \ref{DPboundloss}. This can drastically reduce search time, for example in the case where the distribution $\distribution$ is discrete.
\begin{observation}\label{DPboundloss}
\begin{equation*}
    \min_{r \in \candidaterates}\loss(\thresholdmass_r) = \min_{r \in [0,1]}\loss(\thresholdmass_r).
\end{equation*}
\end{observation}
\begin{proof}
Consider any $r_i \in \candidaterates$. Enough to show that for any $r$ such that $r_{i-} < r < r_{i}$, either $\loss(\thresholdmass_{r_i}) \leq \loss(\thresholdmass_{r})$, or $\loss(\thresholdmass_{r_{i-}}) \leq \loss(\thresholdmass_{r})$, which follows directly from Observation \ref{DPconvex}.
\end{proof}
\subsection{Characterizing the Optimal Randomized PE-Fair Classifier}
In the previous section on Demographic Parity, we considered mass-threshold classifiers $\thresholdmass_t$, that select $\cell_z(t)$, and reject $\cell_z - \cell_z(t)$, for both $z=\advantaged$, and $z=\disadvantaged$. $\thresholdmass_t$ applies the same threshold $t$ to both groups $\advantaged$ and $\disadvantaged$. In this section, we consider groupwise mass-threshold classifiers $\thresholdmass_{t_\advantaged,t_\disadvantaged}$ that apply different thresholds $t_\advantaged$ and $t_\disadvantaged$ to groups $\advantaged$ and $\disadvantaged$ respectively.
\begin{definition}[Groupwise Mass-Threshold Classifier]
A groupwise mass-threshold classifier $\thresholdmass_{t_\advantaged,t_\disadvantaged}$ is defined such that $\accepted(\thresholdmass_{t_\advantaged,t_\disadvantaged}(\cell_z)) = \cell_z(t_z)$, for both $z=\advantaged$, and $z=\disadvantaged$.
\end{definition}
Denote the False Positive rate of a classifier $\classifier$ restricted to a cell $\cell$ by $\text{FPR}(\classifier(\cell))$. Given $r \in (0,1]$, there is a unique classifier $\thresholdmass_{t_\advantaged,t_\disadvantaged}$, such that $\text{FPR}(\thresholdmass_{t_\advantaged,t_\disadvantaged}(\cell_\advantaged)) = \text{FPR}(\thresholdmass_{t_\advantaged,t_\disadvantaged}(\cell_\disadvantaged)) = r$. Denote this classifier by $\fFP$. Given $r=0$, $\thresholdmass_{t_\advantaged,t_\disadvantaged}$ need not be unique as there could exist cells with score $1$. In that case, we define $f_0$ to be the unique groupwise mass-threshold classifier accepting exactly the cells with score $1$.

Now, we answer Problem \ref{prob1} for PE, with the following characterisation result.
\begin{thm}\label{thm:characEO}
There exists $r' \in [0,1]$, such that the optimal PE-fair classifier in $\hypoclass$ is given by the groupwise mass-threshold classifier $f_{r'}$.
\end{thm}
\begin{proof}
We start by first proving Lemma \ref{lemmaEO}, which implies Theorem \ref{thm:characEO}.
\begin{lemma}\label{lemmaEO}
The most accurate PE-fair classifier in $\hypoclass$ with a false positive rate $r \in [0,1]$ is given by $\fFP$.
\end{lemma}
\begin{proof}
Consider an arbitrary classifier $f \neq \fFP$ with false positive rate $r$ for both groups. We will now argue that $\loss(\fFP) < \loss(f)$.

Let the loss of a classifier $f$ restricted to a cell $\cell$ be denoted by $\loss(f(\cell))$.
Enough to show that
\begin{equation*}
    \loss(f(\cell_z)) > \loss(\fFP(\cell_z)), \text{ if }
    \accepted(f(\cell_z)) \neq \accepted(\fFP(\cell_z)).
\end{equation*}
Denote the FP's and FN's of a classifier $f$ on cell $\cell$ by $\text{FP}(\classifier(\cell))$ and $\text{FN}(\classifier(\cell))$ respectively. We know that loss is the sum of false positive and false negative errors. Hence, we need to show that
\begin{equation*}
   \distribution (\text{FP}(f(\cell_z))) + \distribution (\text{FN}(f(\cell_z))) > \distribution (\text{FP}(\fFP(\cell_z))) + \distribution (\text{FN}(\fFP(\cell_z))).
\end{equation*}
Because
\begin{equation*}
   \distribution (\text{FP}(f(\cell_z))) = \distribution (\text{FP}(\fFP(\cell_z))),
\end{equation*}
it is enough to show that
\begin{equation*}
  \distribution (\text{FN}(f(\cell_z))) > \distribution (\text{FN}(\fFP(\cell_z))).
\end{equation*}
Using the fact that the cells of $\cell_z$ are sorted in strictly decreasing order of scores, it is easy to see that
\begin{equation*}
   \score(\accepted(f(\cell_z))) < \score(\accepted(\fFP(\cell_z))).
\end{equation*}
We know that,
\begin{equation*}
   \distribution( \text{FP}(f(\cell_z))) = \distribution(\text{FP}(f(\accepted(f(\cell_z))))) 
\end{equation*}
\begin{equation*}
    = \distribution(\accepted(f(\cell_z))) (1 - \score(\accepted(f(\cell_z)))).
\end{equation*}
Hence,
\begin{equation*}
    \score(\accepted(f(\cell_z))) < \score(\accepted(\fFP(\cell_z)))
\end{equation*}
\begin{equation*}
    \implies
    \distribution(\accepted(f(\cell_z))) < \distribution(\accepted(\fFP(\cell_z)))
\end{equation*}
\begin{equation*}
 \implies  \distribution(\rejected(f(\cell_z))) > \distribution(\rejected(\fFP(\cell_z))).
\end{equation*}
Because the cells of $\cell_z$ are sorted in strictly decreasing order of scores,
\begin{equation*}
    \score(\rejected(f(\cell_z))) > \score(\rejected(\fFP(\cell_z))).
\end{equation*}
Combining the above, we get that
\begin{equation*}
  \distribution (\text{FN}(f(\cell_z))) > \distribution (\text{FN}(\fFP(\cell_z))).
\end{equation*}
This completes the proof of the lemma.
\end{proof}
Armed with this lemma, it is easy to see that the most accurate PE-fair classifier in $\hypoclass$ is given by $f_{r'}$, where $r'$ denotes the optimum false positive rate, i.e.,
\begin{equation*}
        r' \in \argmin_{r \in [0,1]}\loss(\fFP).
\end{equation*}
This completes the proof of the theorem.
\end{proof}
\subsubsection{Finding the Optimal Randomized PE-Fair Classifier by Convex Optimization} After characterizing the optimal classifier, we discuss how to find it  as the solution to a convex optimization problem. The main result is the below theorem.
\begin{thm}\label{thm:convexEO}
$\loss(f_r)$ is a convex, piecewise linear, continuous function over $[0,1]$. Hence,
the optimal classifier $f_{r'}$ in Theorem \ref{thm:characEO}, where
\begin{equation*}
    r' \in \argmin_{r \in [0,1]}\loss(f_r),
\end{equation*}
is given by the solution to a convex optimisation problem.
\end{thm}
We first introduce some terminology that will be needed in the upcoming proofs. Denote the groupwise thresholds of $f_r$ by $r^\advantaged$ and $r^\disadvantaged$ respectively, i.e., $f_r =  \thresholdmass_{r^\advantaged,r^\disadvantaged}$.
\begin{definition}[FP-Boundaries]
Recall the set of score-boundaries $\candidaterates$. We then define the set of FP-boundaries $\candidaterates_{FP}$ as
\begin{equation*}
    \candidaterates_{FP}
    = \{r ~|~ r^\advantaged \in \candidaterates, \text{ or } r^\disadvantaged \in \candidaterates \}.
\end{equation*}
$\candidaterates_{FP}$ essentially consists of all the false positive rates $r$, such that the corresponding groupwise threshold classifier $f_r =  \thresholdmass_{r^\advantaged,r^\disadvantaged}$ has a threshold at a point in the set of score boundaries $\candidaterates$.
\end{definition}
% We now proceed to the proof.
\begin{proof}
We first prove that $\loss(\fFP)$ is piecewise linear and continuous, by proving that the loss function is linear between any two points in the set $\candidaterates_{FP}$.

Consider $r_i \in \candidaterates_{FP} = \{0, r_a, r_b,\ldots,1\}$, and define a corresponding cell $\cell_i$, where
\begin{equation*}
    \cell_i = \accepted(f_{r_{i}}) - \accepted(f_{r_{i-}}) .
\end{equation*}
Consider any $k \geq 0$, such that, $r_{i-} + k \leq r_{i}$. We know that
\begin{equation*}
    \score(\accepted(f_{r_{i-}+k}) - \accepted(f_{r_{i-}})) = \score(\cell_i),
\end{equation*}
because the cell $\cell_i$ is split into two parts at random, and hence both parts have the same score as $\cell_i$.

We know that loss is the (weighted) sum of false positive and false negative errors. Hence,
\begin{equation*}
    \loss(f_{r_{i-}+k}) = 
\end{equation*}
\begin{equation*}
     \loss(f_{r_{i-}}) + \frac{k}{r_i - r_{i-}} \distribution(\cell_i) (1-\score(\cell_i)) - \frac{k}{r_i - r_{i-}} \distribution(\cell_i) \score(\cell_i) =
\end{equation*}
\begin{equation*}
    \loss(f_{r_{i-}}) + \frac{k}{r_i - r_{i-}} \distribution(\cell_i) (1 - 2\score(\cell_i)) =
     \loss(f_{r_{i-}}) + kc(1 - 2\score(\cell_i)),
\end{equation*}
for some $c>0$. Looking at the form of the expression of the loss as above, it is easy to see that the loss function is linear in the interval $[r_{i-},r_{i}]$, for any $r_i$, and is therefore piecewise linear and continuous over $[0,1]$.

We now proceed to prove that the loss function is convex. We begin by observing the following.
\begin{observation}\label{EOconvex}
Consider the interval $[r_{i-},r_{i}]$, for any $r_i \in \candidaterates_{FP}$. For any point in this interval, the slope of the loss function $\loss(\fFP)$ is $c(1-2\score(\cell_i))$ for some $c>0$, and hence
\begin{enumerate}
    \item $\loss(\fFP)$ is strictly decreasing if $\score(\cell_i) > \frac{1}{2}$, and the decrease is steeper if the score is higher. 
    \item $\loss(\fFP)$ is strictly increasing if $\score(\cell_i) < \frac{1}{2}$, and the increase is steeper if the score is lower.
    \item $\loss(\fFP)$ is constant if $\score(\cell_i) = \frac{1}{2}$. 
\end{enumerate}
\end{observation}
Because the component cells of $\cell_z$ are sorted in strictly descending order of scores, it is easy to see that, for every $r_i$, $\score(\cell_{i}) < \score(\cell_{i-})$.
Hence, there exists a cell $\cell_{i'}$, such that for every $r_j, r_k \in \candidaterates_{FP}$ such that $r_j < r_{i'} < r_k$,  one of the following cases hold.
\begin{enumerate}
    \item $\score(\cell_{j}) > \frac{1}{2} > \score(\cell_{i'}) > \score(\cell_{k})$.
    \item $\score(\cell_{j}) > \score(\cell_{i'}) > \frac{1}{2} > \score(\cell_{k})$.
    \item $\score(\cell_{j}) > \score(\cell_{i'}) = \frac{1}{2} > \score(\cell_{k})$.
\end{enumerate}
These are the same 3 cases as in the proof of Theorem \ref{thm:convexDP}. Hence, in each case, we apply the same reasoning as in the proof of Theorem \ref{thm:convexDP} to see that the loss function is convex. 

This concludes the proof. 
\end{proof}
\begin{remark}
We note that the optima is unique if there exists no cell $\cell_i$, such that $\score(\cell_i) = \frac{1}{2}$.
\end{remark}
In addition, to find the optimum false positive rate $r'$, we just need to evaluate $\loss(f_r)$ for all $r \in \candidaterates_{FP}$, instead of for all $r \in [0,1]$, as a corollary of Observation \ref{EOboundloss}. This could drastically reduce search time, for example in the case where the distribution $\distribution$ is discrete.
\begin{observation}\label{EOboundloss}
\begin{equation*}
    \min_{r \in \candidaterates_{FP}}\loss(f_r) = \min_{r \in [0,1]}\loss(f_r).
\end{equation*}
\end{observation}
\begin{proof}
Consider $r_i \in \candidaterates_{FP}$. Enough to show that for any $r$ such that $r_{i-} < r < r_{i}$, either $\loss(f_{r_i}) \leq \loss(f_{r})$, or $\loss(f_{r_{i-}}) \leq \loss(f_{r})$, which follows directly from Observation \ref{EOconvex}.
\end{proof}
\section{Optimal Randomized Fair Representation}\label{rep}
Many previous works have given heuristics to construct fair representations, but in most cases they do not provide accuracy guarantees for the representation. Hence, the cost of fairness for the representations they give may be high, resulting in a loss in utility. Here we construct a fair representation that has $0$ cost of fairness, implying optimal accuracy. The representation we construct follows naturally from our previous analysis of optimal fair randomized classifier.
\subsection{Characterizing the Optimal Randomized DP-Fair Representation}
Now, we answer Problem \ref{prob2} for DP, with the following result.
\begin{thm}
There exists a DP-fair representation $R_{\text{DP}}$, such that cost of fairness for $\RDP$ is $0$.
\end{thm}

\begin{proof}

We prove this by explicitly constructing such a representation. Recall that we had previously defined the set of score boundaries $\candidaterates$. As in the proof of Theorem \ref{thm:convexDP}, consider any $r_i \in \candidaterates = \{0, r_a, r_b,\ldots,1\}$, and define a corresponding cell $\cell_i$, where
\begin{equation*}
   \cell_i = \accepted(\thresholdmass_{r_i}) - \accepted(\thresholdmass_{r_{i-}}),
\end{equation*}
and $r_{i-}$ denotes the element in $\candidaterates$ preceding $r_{i}$.

Consider a representation $\RDP$, which maps each $\cell_i$ to a distinct $\Tilde{x}_i \in \repspace$, such that given distinct $r_i, r_j \in \candidaterates$, $\Tilde{x}_i \neq \Tilde{x}_j$. We now proceed to prove that $\RDP$ is DP-fair and has $0$ cost of fairness.

Each classifier $f$ over $\RDP$ either accepts or rejects $\Tilde{x}_i$. Hence, each $f$ either accepts or rejects $\cell_i$. Denote the set of all $i \in \candidaterates$, such that $f$ accepts $\cell_i$, by $\candidaterates_f$. 
It is easy to see that for both groups $\advantaged$, and $\disadvantaged$, selection rate is given by
\begin{equation*}
    \sum_{i \in \candidaterates_f} (r_{i} - r_{i-}).
\end{equation*}
Hence, $f$ satisfies demographic parity, and $\RDP$ is DP-fair.

It is easy to see that for every $r \in \candidaterates$, the threshold classifier $\thresholdmass_r$ either accepts or rejects $\cell_i$.
Hence, for every $r \in \candidaterates$, the threshold classifier $\thresholdmass_r$ is expressible over $\RDP$, i.e., $\thresholdmass_r \in \hypoclass_{\RDP}$. Now, let us say we are given an optimal fair classifier $\thresholdmass_r'$, as in Theorem \ref{thm:characDP}. Then, there exists $r \in \candidaterates$, such $\loss(\thresholdmass_r) = \loss(\thresholdmass_r')$, which follows from Observation \ref{DPboundloss}. Since $r \in \candidaterates$, $\thresholdmass_r \in \hypoclass_{\RDP}$, and the cost of fairness for $\RDP$ is $0$.
%the optimal classifier in Theorem \ref{thm:characDP} is expressible over $\RDP$,

This concludes the proof.
\end{proof}
\subsection{Characterizing the Optimal Randomized PE-Fair Representation}
Now, we answer Problem \ref{prob2} for EO, with the following result.
\begin{thm}
There exists a PE-fair representation $R_{\text{PE}}$, such that cost of fairness for $R_{\text{PE}}$ is $0$.
\end{thm}

\begin{proof}
We prove this by explicitly constructing such a representation. Consider $r_i \in \candidaterates_{FP} = \{0, r_a, r_b,\ldots,1\}$. As in the proof of Theorem \ref{thm:convexEO}, define a corresponding cell $\cell_i$, where
\begin{equation*}
    \cell_i = \accepted(f_{r_{i}}) - \accepted(f_{r_{i-}}) .
\end{equation*}
Now, define  representation $\RFP$, which maps each $\cell_i$ to a distinct $\Tilde{x}_i \in \repspace$, such that given distinct $r_i, r_j \in \candidaterates_{FP}$, $\Tilde{x}_i \neq \Tilde{x}_j$. We now proceed to prove that $\RFP$ is PE-fair and has $0$ cost of fairness.

Each classifier $f$ over $\RFP$ either accepts or rejects $\Tilde{x}_i$. Hence, each $f$ either accepts or rejects $\cell_i$. Denote the set of all $i \in \candidaterates_{FP}$, such that $f$ accepts $\cell_i$, by $\candidaterates_f^{FP}$. 
It is easy to see that for both groups $\advantaged,\disadvantaged$, false positive rate is given by
\begin{equation*}
    \sum_{i \in \candidaterates_f^{FP}} (r_{i} - r_{i-}).
\end{equation*}
Hence, $f$ satisfies equal false positive rates, and $\RFP$ is PE-fair.

It is easy to see that for every $r \in \candidaterates_{FP}$, groupwise threshold classifier $\fFP = \thresholdmass_{r^\advantaged,r^\disadvantaged}$ either accepts or rejects $\cell_i$.
Hence, for every $r \in \candidaterates_{FP}$, threshold classifier $f_r$ is expressible over $\RFP$, i.e., $f_r \in \hypoclass_{\RFP}$. Now, let us say we are given an optimal fair classifier $f_r'$, as in Theorem \ref{thm:characEO}. Then, there exists $r \in \candidaterates_{FP}$, such that $\loss(f_r) = \loss(f_r')$, which follows from Observation \ref{EOboundloss}. Since $r \in \candidaterates_{FP}$, $f_r \in \hypoclass_{\RFP}$, and the cost of fairness for $\RFP$ is $0$.

This concludes the proof.
\end{proof}
\section{Conclusion and Future Work}\label{conclusion}
In this paper, we examined the power of randomization in the problems of fair classification and representation. In particular, we mathematically characterized the optimal randomized DP-fair, PE-fair, and EO-fair classifiers in the group aware setting, and showed that they are in general more accurate than their deterministic counterparts. In addition, we also showed how the classifiers that we characterize can be obtained as a solution to a convex optimization problem. Extending our ideas for randomized fair classification, we construct DP-fair, PE-fair, and EO-fair representations such that the optimal classifier on this representation incurs no \emph{cost of fair representation}, i.e., no accuracy loss compared to the optimal randomized fair classifier on the original data distribution. This improves upon previous work on fair representations that primarily focuses on DP-fairness and gives no or weak provable guarantees for the accuracy loss.

Some directions for further work include extending our results for binary classification to multi-class classification and regression. Another direction could be to look at relaxed or approximate versions of the fairness notions we considered, or to look at other popular notions of fairness. It would also be valuable to experimentally validate our theoretical claims.

\begin{acks}
\textbf{Funding/Support:} The authors were employed by Microsoft Research India during the course of this work.
\end{acks}

\bibliographystyle{ACM-Reference-Format}
\bibliography{sample-base}


\begin{thebibliography}{39}

%%% the Web address.

\ifx \showCODEN    \undefined \def \showCODEN     #1{\unskip}     \fi
\ifx \showDOI      \undefined \def \showDOI       #1{#1}\fi
\ifx \showISBNx    \undefined \def \showISBNx     #1{\unskip}     \fi
\ifx \showISBNxiii \undefined \def \showISBNxiii  #1{\unskip}     \fi
\ifx \showISSN     \undefined \def \showISSN      #1{\unskip}     \fi
\ifx \showLCCN     \undefined \def \showLCCN      #1{\unskip}     \fi
\ifx \shownote     \undefined \def \shownote      #1{#1}          \fi
\ifx \showarticletitle \undefined \def \showarticletitle #1{#1}   \fi
\ifx \showURL      \undefined \def \showURL       {\relax}        \fi
% invisible to TeX
\providecommand\bibfield[2]{#2}
\providecommand\bibinfo[2]{#2}
\providecommand\natexlab[1]{#1}
\providecommand\showeprint[2][]{arXiv:#2}

\bibitem[Agarwal et~al\mbox{.}(2018)]%
        {agarwal2018reductions}
\bibfield{author}{\bibinfo{person}{Alekh Agarwal}, \bibinfo{person}{Alina Beygelzimer}, \bibinfo{person}{Miroslav Dud{\'\i}k}, \bibinfo{person}{John Langford}, {and} \bibinfo{person}{Hanna Wallach}.} \bibinfo{year}{2018}\natexlab{}.
\newblock \showarticletitle{A reductions approach to fair classification}. In \bibinfo{booktitle}{\emph{International Conference on Machine Learning}}. PMLR, \bibinfo{pages}{60--69}.
\newblock


\bibitem[Agarwal(2021a)]%
        {agarwal2020trade}
\bibfield{author}{\bibinfo{person}{Sushant Agarwal}.} \bibinfo{year}{2021}\natexlab{a}.
\newblock \showarticletitle{Trade-Offs between Fairness and Interpretability in Machine Learning}. In \bibinfo{booktitle}{\emph{IJCAI 2021 Workshop on AI for Social Good}}.
\newblock


\bibitem[Agarwal(2021b)]%
        {agarwal2021trade}
\bibfield{author}{\bibinfo{person}{Sushant Agarwal}.} \bibinfo{year}{2021}\natexlab{b}.
\newblock \showarticletitle{Trade-Offs between Fairness and Privacy in Machine Learning}. In \bibinfo{booktitle}{\emph{IJCAI 2021 Workshop on AI for Social Good}}.
\newblock


\bibitem[Barocas et~al\mbox{.}(2019)]%
        {barocas2019fairmlbook}
\bibfield{author}{\bibinfo{person}{Solon Barocas}, \bibinfo{person}{Moritz Hardt}, {and} \bibinfo{person}{Arvind Narayanan}.} \bibinfo{year}{2019}\natexlab{}.
\newblock \bibinfo{booktitle}{\emph{Fairness and Machine Learning}}.
\newblock \bibinfo{publisher}{fairmlbook.org}.
\newblock
\newblock
\shownote{\url{http://www.fairmlbook.org}}.


\bibitem[Barocas and Selbst(2016)]%
        {barocas2016big}
\bibfield{author}{\bibinfo{person}{Solon Barocas} {and} \bibinfo{person}{Andrew~D Selbst}.} \bibinfo{year}{2016}\natexlab{}.
\newblock \showarticletitle{Big Data's Disparate Impact}.
\newblock \bibinfo{journal}{\emph{California Law Review}}  \bibinfo{volume}{104} (\bibinfo{year}{2016}), \bibinfo{pages}{671}.
\newblock
\urldef\tempurl%
\url{http://lawcat.berkeley.edu/record/1127463}
\showURL{%
\tempurl}


\bibitem[Bechavod et~al\mbox{.}(2020)]%
        {bechavod2020metric}
\bibfield{author}{\bibinfo{person}{Yahav Bechavod}, \bibinfo{person}{Christopher Jung}, {and} \bibinfo{person}{Zhiwei~Steven Wu}.} \bibinfo{year}{2020}\natexlab{}.
\newblock \showarticletitle{Metric-Free Individual Fairness in Online Learning}. In \bibinfo{booktitle}{\emph{Advances in Neural Information Processing Systems 33: Annual Conference on Neural Information Processing Systems 2020, NeurIPS 2020, December 6-12, 2020, virtual}}.
\newblock
\urldef\tempurl%
\url{https://proceedings.neurips.cc/paper/2020/hash/80b618ebcac7aa97a6dac2ba65cb7e36-Abstract.html}
\showURL{%
\tempurl}


\bibitem[Celis et~al\mbox{.}(2020)]%
        {celis2020classification}
\bibfield{author}{\bibinfo{person}{L.~Elisa Celis}, \bibinfo{person}{Lingxiao Huang}, \bibinfo{person}{Vijay Keswani}, {and} \bibinfo{person}{Nisheeth~K. Vishnoi}.} \bibinfo{year}{2020}\natexlab{}.
\newblock \bibinfo{title}{Classification with Fairness Constraints: A Meta-Algorithm with Provable Guarantees}.
\newblock
\newblock
\showeprint[arxiv]{1806.06055}~[cs.LG]


\bibitem[Choi et~al\mbox{.}(2020)]%
        {choi2020fair}
\bibfield{author}{\bibinfo{person}{Kristy Choi}, \bibinfo{person}{Aditya Grover}, \bibinfo{person}{Trisha Singh}, \bibinfo{person}{Rui Shu}, {and} \bibinfo{person}{Stefano Ermon}.} \bibinfo{year}{2020}\natexlab{}.
\newblock \showarticletitle{Fair Generative Modeling via Weak Supervision}. In \bibinfo{booktitle}{\emph{Proceedings of the 37th International Conference on Machine Learning, {ICML} 2020, 13-18 July 2020, Virtual Event}} \emph{(\bibinfo{series}{Proceedings of Machine Learning Research}, Vol.~\bibinfo{volume}{119})}. \bibinfo{publisher}{{PMLR}}, \bibinfo{pages}{1887--1898}.
\newblock
\urldef\tempurl%
\url{http://proceedings.mlr.press/v119/choi20a.html}
\showURL{%
\tempurl}


\bibitem[Chouldechova and Roth(2020)]%
        {chouldechova2020snapshot}
\bibfield{author}{\bibinfo{person}{Alexandra Chouldechova} {and} \bibinfo{person}{Aaron Roth}.} \bibinfo{year}{2020}\natexlab{}.
\newblock \showarticletitle{A Snapshot of the Frontiers of Fairness in Machine Learning}.
\newblock \bibinfo{journal}{\emph{Commun. ACM}} \bibinfo{volume}{63}, \bibinfo{number}{5} (\bibinfo{year}{2020}), \bibinfo{pages}{82–89}.
\newblock
\urldef\tempurl%
\url{https://doi.org/10.1145/3376898}
\showURL{%
\tempurl}


\bibitem[Chzhen et~al\mbox{.}(2019)]%
        {DBLP:conf/nips/ChzhenDHOP19}
\bibfield{author}{\bibinfo{person}{Evgenii Chzhen}, \bibinfo{person}{Christophe Denis}, \bibinfo{person}{Mohamed Hebiri}, \bibinfo{person}{Luca Oneto}, {and} \bibinfo{person}{Massimiliano Pontil}.} \bibinfo{year}{2019}\natexlab{}.
\newblock \showarticletitle{Leveraging Labeled and Unlabeled Data for Consistent Fair Binary Classification}. In \bibinfo{booktitle}{\emph{Advances in Neural Information Processing Systems 32: Annual Conference on Neural Information Processing Systems 2019, NeurIPS 2019, December 8-14, 2019, Vancouver, BC, Canada}}, \bibfield{editor}{\bibinfo{person}{Hanna~M. Wallach}, \bibinfo{person}{Hugo Larochelle}, \bibinfo{person}{Alina Beygelzimer}, \bibinfo{person}{Florence d'Alch{\'{e}}{-}Buc}, \bibinfo{person}{Emily~B. Fox}, {and} \bibinfo{person}{Roman Garnett}} (Eds.). \bibinfo{pages}{12739--12750}.
\newblock
\urldef\tempurl%
\url{https://proceedings.neurips.cc/paper/2019/hash/ba51e6158bcaf80fd0d834950251e693-Abstract.html}
\showURL{%
\tempurl}


\bibitem[Chzhen et~al\mbox{.}(2020)]%
        {DBLP:conf/nips/ChzhenDHOP20}
\bibfield{author}{\bibinfo{person}{Evgenii Chzhen}, \bibinfo{person}{Christophe Denis}, \bibinfo{person}{Mohamed Hebiri}, \bibinfo{person}{Luca Oneto}, {and} \bibinfo{person}{Massimiliano Pontil}.} \bibinfo{year}{2020}\natexlab{}.
\newblock \showarticletitle{Fair regression with Wasserstein barycenters}. In \bibinfo{booktitle}{\emph{Advances in Neural Information Processing Systems 33: Annual Conference on Neural Information Processing Systems 2020, NeurIPS 2020, December 6-12, 2020, virtual}}, \bibfield{editor}{\bibinfo{person}{Hugo Larochelle}, \bibinfo{person}{Marc'Aurelio Ranzato}, \bibinfo{person}{Raia Hadsell}, \bibinfo{person}{Maria{-}Florina Balcan}, {and} \bibinfo{person}{Hsuan{-}Tien Lin}} (Eds.).
\newblock
\urldef\tempurl%
\url{https://proceedings.neurips.cc/paper/2020/hash/51cdbd2611e844ece5d80878eb770436-Abstract.html}
\showURL{%
\tempurl}


\bibitem[Corbett-Davies and Goel(2018)]%
        {corbettdavies2018measure}
\bibfield{author}{\bibinfo{person}{Sam Corbett-Davies} {and} \bibinfo{person}{Sharad Goel}.} \bibinfo{year}{2018}\natexlab{}.
\newblock \bibinfo{title}{The Measure and Mismeasure of Fairness: A Critical Review of Fair Machine Learning}.
\newblock
\newblock
\showeprint[arxiv]{1808.00023}~[cs.CY]


\bibitem[Corbett-Davies et~al\mbox{.}(2017)]%
        {corbettcost}
\bibfield{author}{\bibinfo{person}{Sam Corbett-Davies}, \bibinfo{person}{Emma Pierson}, \bibinfo{person}{Avi Feller}, \bibinfo{person}{Sharad Goel}, {and} \bibinfo{person}{Aziz Huq}.} \bibinfo{year}{2017}\natexlab{}.
\newblock \showarticletitle{Algorithmic Decision Making and the Cost of Fairness}. In \bibinfo{booktitle}{\emph{Proceedings of the 23rd ACM SIGKDD International Conference on Knowledge Discovery and Data Mining}} (Halifax, NS, Canada) \emph{(\bibinfo{series}{KDD '17})}. \bibinfo{publisher}{Association for Computing Machinery}, \bibinfo{address}{New York, NY, USA}, \bibinfo{pages}{797–806}.
\newblock
\showISBNx{9781450348874}
\urldef\tempurl%
\url{https://doi.org/10.1145/3097983.3098095}
\showDOI{\tempurl}


\bibitem[Cotter et~al\mbox{.}(2019)]%
        {NEURIPS2019_5fc34ed3}
\bibfield{author}{\bibinfo{person}{Andrew Cotter}, \bibinfo{person}{Maya Gupta}, {and} \bibinfo{person}{Harikrishna Narasimhan}.} \bibinfo{year}{2019}\natexlab{}.
\newblock \showarticletitle{On Making Stochastic Classifiers Deterministic}. In \bibinfo{booktitle}{\emph{Advances in Neural Information Processing Systems}}, \bibfield{editor}{\bibinfo{person}{H.~Wallach}, \bibinfo{person}{H.~Larochelle}, \bibinfo{person}{A.~Beygelzimer}, \bibinfo{person}{F.~d\textquotesingle Alch\'{e}-Buc}, \bibinfo{person}{E.~Fox}, {and} \bibinfo{person}{R.~Garnett}} (Eds.), Vol.~\bibinfo{volume}{32}. \bibinfo{publisher}{Curran Associates, Inc.}
\newblock
\urldef\tempurl%
\url{https://proceedings.neurips.cc/paper/2019/file/5fc34ed307aac159a30d81181c99847e-Paper.pdf}
\showURL{%
\tempurl}


\bibitem[Creager et~al\mbox{.}(2019)]%
        {creager2019flexibly}
\bibfield{author}{\bibinfo{person}{Elliot Creager}, \bibinfo{person}{David Madras}, \bibinfo{person}{Joern-Henrik Jacobsen}, \bibinfo{person}{Marissa Weis}, \bibinfo{person}{Kevin Swersky}, \bibinfo{person}{Toniann Pitassi}, {and} \bibinfo{person}{Richard Zemel}.} \bibinfo{year}{2019}\natexlab{}.
\newblock \showarticletitle{Flexibly Fair Representation Learning by Disentanglement}. In \bibinfo{booktitle}{\emph{Proceedings of the 36th International Conference on Machine Learning}} \emph{(\bibinfo{series}{Proceedings of Machine Learning Research}, Vol.~\bibinfo{volume}{97})}. \bibinfo{publisher}{PMLR}, \bibinfo{pages}{1436--1445}.
\newblock
\urldef\tempurl%
\url{https://proceedings.mlr.press/v97/creager19a.html}
\showURL{%
\tempurl}


\bibitem[del Barrio et~al\mbox{.}(2018)]%
        {delbarrio2018obtaining}
\bibfield{author}{\bibinfo{person}{Eustasio del Barrio}, \bibinfo{person}{Fabrice Gamboa}, \bibinfo{person}{Paula Gordaliza}, {and} \bibinfo{person}{Jean-Michel Loubes}.} \bibinfo{year}{2018}\natexlab{}.
\newblock \bibinfo{title}{Obtaining fairness using optimal transport theory}.
\newblock
\newblock
\showeprint[arxiv]{1806.03195}~[math.ST]


\bibitem[Dwork et~al\mbox{.}(2012)]%
        {dwork2012fairness}
\bibfield{author}{\bibinfo{person}{Cynthia Dwork}, \bibinfo{person}{Moritz Hardt}, \bibinfo{person}{Toniann Pitassi}, \bibinfo{person}{Omer Reingold}, {and} \bibinfo{person}{Richard Zemel}.} \bibinfo{year}{2012}\natexlab{}.
\newblock \showarticletitle{Fairness through Awareness}. In \bibinfo{booktitle}{\emph{Proceedings of the 3rd Innovations in Theoretical Computer Science (ITCS) Conference}}. \bibinfo{publisher}{Association for Computing Machinery}, \bibinfo{pages}{214–226}.
\newblock
\urldef\tempurl%
\url{https://doi.org/10.1145/2090236.2090255}
\showURL{%
\tempurl}


\bibitem[Feldman et~al\mbox{.}(2015)]%
        {feldman2015certifying}
\bibfield{author}{\bibinfo{person}{Michael Feldman}, \bibinfo{person}{Sorelle~A. Friedler}, \bibinfo{person}{John Moeller}, \bibinfo{person}{Carlos Scheidegger}, {and} \bibinfo{person}{Suresh Venkatasubramanian}.} \bibinfo{year}{2015}\natexlab{}.
\newblock \showarticletitle{Certifying and Removing Disparate Impact}. In \bibinfo{booktitle}{\emph{Proceedings of the 21th ACM SIGKDD International Conference on Knowledge Discovery and Data Mining}} (Sydney, NSW, Australia) \emph{(\bibinfo{series}{KDD '15})}. \bibinfo{publisher}{Association for Computing Machinery}, \bibinfo{pages}{259–268}.
\newblock
\showISBNx{9781450336642}
\urldef\tempurl%
\url{https://doi.org/10.1145/2783258.2783311}
\showURL{%
\tempurl}


\bibitem[Ferber et~al\mbox{.}(2021)]%
        {ferber2021differentiable}
\bibfield{author}{\bibinfo{person}{Aaron Ferber}, \bibinfo{person}{Umang Gupta}, \bibinfo{person}{Greg~Ver Steeg}, {and} \bibinfo{person}{Bistra Dilkina}.} \bibinfo{year}{2021}\natexlab{}.
\newblock \showarticletitle{Differentiable Optimal Adversaries for Learning Fair Representations}. In \bibinfo{booktitle}{\emph{IJCAI 2021 Workshop on AI for Social Good}}.
\newblock


\bibitem[Fleisher(2021)]%
        {fleisher2021what}
\bibfield{author}{\bibinfo{person}{Will Fleisher}.} \bibinfo{year}{2021}\natexlab{}.
\newblock \bibinfo{booktitle}{\emph{What's Fair about Individual Fairness?}}
\newblock \bibinfo{publisher}{Association for Computing Machinery}, \bibinfo{pages}{480–490}.
\newblock


\bibitem[Hardt et~al\mbox{.}(2016)]%
        {hardt2016equality}
\bibfield{author}{\bibinfo{person}{Moritz Hardt}, \bibinfo{person}{Eric Price}, \bibinfo{person}{Eric Price}, {and} \bibinfo{person}{Nati Srebro}.} \bibinfo{year}{2016}\natexlab{}.
\newblock \showarticletitle{Equality of Opportunity in Supervised Learning}. In \bibinfo{booktitle}{\emph{Advances in Neural Information Processing Systems}}, Vol.~\bibinfo{volume}{29}.
\newblock
\urldef\tempurl%
\url{https://proceedings.neurips.cc/paper/2016/file/9d2682367c3935defcb1f9e247a97c0d-Paper.pdf}
\showURL{%
\tempurl}


\bibitem[Johndrow and Lum(2017)]%
        {johndrow2017algorithm}
\bibfield{author}{\bibinfo{person}{James~E. Johndrow} {and} \bibinfo{person}{Kristian Lum}.} \bibinfo{year}{2017}\natexlab{}.
\newblock \bibinfo{title}{An algorithm for removing sensitive information: application to race-independent recidivism prediction}.
\newblock
\newblock
\showeprint[arxiv]{1703.04957}~[stat.AP]


\bibitem[Kamiran and Calders(2012)]%
        {kamiran2012data}
\bibfield{author}{\bibinfo{person}{Faisal Kamiran} {and} \bibinfo{person}{Toon Calders}.} \bibinfo{year}{2012}\natexlab{}.
\newblock \showarticletitle{Data preprocessing techniques for classification without discrimination}.
\newblock \bibinfo{journal}{\emph{Knowledge and Information Systems}} \bibinfo{volume}{33}, \bibinfo{number}{1} (\bibinfo{year}{2012}), \bibinfo{pages}{1--33}.
\newblock


\bibitem[Kleinberg and Mullainathan(2019)]%
        {kleinberg2019simplicity}
\bibfield{author}{\bibinfo{person}{Jon Kleinberg} {and} \bibinfo{person}{Sendhil Mullainathan}.} \bibinfo{year}{2019}\natexlab{}.
\newblock \bibinfo{title}{Simplicity Creates Inequity: Implications for Fairness, Stereotypes, and Interpretability}.
\newblock
\newblock
\showeprint[arxiv]{1809.04578}~[cs.LG]


\bibitem[Kleinberg et~al\mbox{.}(2016)]%
        {kleinberg2016inherent}
\bibfield{author}{\bibinfo{person}{Jon Kleinberg}, \bibinfo{person}{Sendhil Mullainathan}, {and} \bibinfo{person}{Manish Raghavan}.} \bibinfo{year}{2016}\natexlab{}.
\newblock \showarticletitle{Inherent trade-offs in the fair determination of risk scores}.
\newblock \bibinfo{journal}{\emph{arXiv preprint arXiv:1609.05807}} (\bibinfo{year}{2016}).
\newblock


\bibitem[Louizos et~al\mbox{.}(2017)]%
        {louizos2017variational}
\bibfield{author}{\bibinfo{person}{Christos Louizos}, \bibinfo{person}{Kevin Swersky}, \bibinfo{person}{Yujia Li}, \bibinfo{person}{Max Welling}, {and} \bibinfo{person}{Richard Zemel}.} \bibinfo{year}{2017}\natexlab{}.
\newblock \bibinfo{title}{The Variational Fair Autoencoder}.
\newblock
\newblock
\showeprint[arxiv]{1511.00830}~[stat.ML]


\bibitem[Madras et~al\mbox{.}(2018)]%
        {madras2018learning}
\bibfield{author}{\bibinfo{person}{David Madras}, \bibinfo{person}{Elliot Creager}, \bibinfo{person}{Toniann Pitassi}, {and} \bibinfo{person}{Richard Zemel}.} \bibinfo{year}{2018}\natexlab{}.
\newblock \showarticletitle{Learning adversarially fair and transferable representations}. In \bibinfo{booktitle}{\emph{International Conference on Machine Learning}}. PMLR, \bibinfo{pages}{3384--3393}.
\newblock


\bibitem[McNamara et~al\mbox{.}(2019)]%
        {mcnamara201costs}
\bibfield{author}{\bibinfo{person}{Daniel McNamara}, \bibinfo{person}{Cheng~Soon Ong}, {and} \bibinfo{person}{Robert~C. Williamson}.} \bibinfo{year}{2019}\natexlab{}.
\newblock \showarticletitle{Costs and Benefits of Fair Representation Learning}. In \bibinfo{booktitle}{\emph{Proceedings of the 2019 AAAI/ACM Conference on AI, Ethics, and Society}} \emph{(\bibinfo{series}{AIES '19})}. \bibinfo{publisher}{Association for Computing Machinery}, \bibinfo{pages}{263–270}.
\newblock
\showISBNx{9781450363242}
\urldef\tempurl%
\url{https://doi.org/10.1145/3306618.3317964}
\showDOI{\tempurl}


\bibitem[Menon and Williamson(2018)]%
        {pmlr-v81-menon18a}
\bibfield{author}{\bibinfo{person}{Aditya~Krishna Menon} {and} \bibinfo{person}{Robert~C Williamson}.} \bibinfo{year}{2018}\natexlab{}.
\newblock \showarticletitle{The cost of fairness in binary classification}. In \bibinfo{booktitle}{\emph{Proceedings of the 1st Conference on Fairness, Accountability and Transparency}} \emph{(\bibinfo{series}{Proceedings of Machine Learning Research}, Vol.~\bibinfo{volume}{81})}, \bibfield{editor}{\bibinfo{person}{Sorelle~A. Friedler} {and} \bibinfo{person}{Christo Wilson}} (Eds.). \bibinfo{publisher}{PMLR}, \bibinfo{pages}{107--118}.
\newblock
\urldef\tempurl%
\url{https://proceedings.mlr.press/v81/menon18a.html}
\showURL{%
\tempurl}


\bibitem[Papakyriakopoulos et~al\mbox{.}(2020)]%
        {papakyriakopoulos2020bias}
\bibfield{author}{\bibinfo{person}{Orestis Papakyriakopoulos}, \bibinfo{person}{Simon Hegelich}, \bibinfo{person}{Juan Carlos~Medina Serrano}, {and} \bibinfo{person}{Fabienne Marco}.} \bibinfo{year}{2020}\natexlab{}.
\newblock \showarticletitle{Bias in Word Embeddings}. In \bibinfo{booktitle}{\emph{Proceedings of the 2020 Conference on Fairness, Accountability, and Transparency}} \emph{(\bibinfo{series}{FAT* '20})}. \bibinfo{publisher}{Association for Computing Machinery}, \bibinfo{address}{New York, NY, USA}, \bibinfo{pages}{446–457}.
\newblock
\showISBNx{9781450369367}
\urldef\tempurl%
\url{https://doi.org/10.1145/3351095.3372843}
\showURL{%
\tempurl}


\bibitem[Petersen et~al\mbox{.}(2021)]%
        {petersen2021post}
\bibfield{author}{\bibinfo{person}{Felix Petersen}, \bibinfo{person}{Debarghya Mukherjee}, \bibinfo{person}{Yuekai Sun}, {and} \bibinfo{person}{Mikhail Yurochkin}.} \bibinfo{year}{2021}\natexlab{}.
\newblock \showarticletitle{Post-processing for Individual Fairness}.
\newblock \bibinfo{journal}{\emph{Advances in Neural Information Processing Systems}}  \bibinfo{volume}{34} (\bibinfo{year}{2021}).
\newblock


\bibitem[Ruoss et~al\mbox{.}(2020)]%
        {ruoss2020learning}
\bibfield{author}{\bibinfo{person}{Anian Ruoss}, \bibinfo{person}{Mislav Balunovi{\'c}}, \bibinfo{person}{Marc Fischer}, {and} \bibinfo{person}{Martin Vechev}.} \bibinfo{year}{2020}\natexlab{}.
\newblock \showarticletitle{Learning certified individually fair representations}.
\newblock \bibinfo{journal}{\emph{arXiv preprint arXiv:2002.10312}} (\bibinfo{year}{2020}).
\newblock


\bibitem[Silvia et~al\mbox{.}(2020)]%
        {Silvia_Ray_Tom_Aldo_Heinrich_John_2020}
\bibfield{author}{\bibinfo{person}{Chiappa Silvia}, \bibinfo{person}{Jiang Ray}, \bibinfo{person}{Stepleton Tom}, \bibinfo{person}{Pacchiano Aldo}, \bibinfo{person}{Jiang Heinrich}, {and} \bibinfo{person}{Aslanides John}.} \bibinfo{year}{2020}\natexlab{}.
\newblock \showarticletitle{A General Approach to Fairness with Optimal Transport}.
\newblock \bibinfo{journal}{\emph{Proceedings of the AAAI Conference on Artificial Intelligence}} \bibinfo{volume}{34}, \bibinfo{number}{04} (\bibinfo{date}{Apr.} \bibinfo{year}{2020}), \bibinfo{pages}{3633--3640}.
\newblock
\urldef\tempurl%
\url{https://doi.org/10.1609/aaai.v34i04.5771}
\showDOI{\tempurl}


\bibitem[Tan et~al\mbox{.}(2020)]%
        {tan2020learning}
\bibfield{author}{\bibinfo{person}{Zilong Tan}, \bibinfo{person}{Samuel Yeom}, \bibinfo{person}{Matt Fredrikson}, {and} \bibinfo{person}{Ameet Talwalkar}.} \bibinfo{year}{2020}\natexlab{}.
\newblock \showarticletitle{Learning Fair Representations for Kernel Models}. In \bibinfo{booktitle}{\emph{Proceedings of the Twenty Third International Conference on Artificial Intelligence and Statistics}} \emph{(\bibinfo{series}{Proceedings of Machine Learning Research}, Vol.~\bibinfo{volume}{108})}. \bibinfo{publisher}{PMLR}, \bibinfo{pages}{155--166}.
\newblock
\urldef\tempurl%
\url{https://proceedings.mlr.press/v108/tan20a.html}
\showURL{%
\tempurl}


\bibitem[Zafar et~al\mbox{.}(2017a)]%
        {zafar2017fairnessbeyond}
\bibfield{author}{\bibinfo{person}{Muhammad~Bilal Zafar}, \bibinfo{person}{Isabel Valera}, \bibinfo{person}{Manuel Gomez~Rodriguez}, {and} \bibinfo{person}{Krishna~P. Gummadi}.} \bibinfo{year}{2017}\natexlab{a}.
\newblock \showarticletitle{Fairness Beyond Disparate Treatment and Disparate Impact: Learning Classification without Disparate Mistreatment}. In \bibinfo{booktitle}{\emph{Proceedings of the 26th International Conference on World Wide Web (WWW)}} \emph{(\bibinfo{series}{WWW'17})}. \bibinfo{pages}{1171–1180}.
\newblock


\bibitem[Zafar et~al\mbox{.}(2017b)]%
        {zafar2017fairnessconstraints}
\bibfield{author}{\bibinfo{person}{Muhammad~Bilal Zafar}, \bibinfo{person}{Isabel Valera}, \bibinfo{person}{Manuel~Gomez Rogriguez}, {and} \bibinfo{person}{Krishna~P Gummadi}.} \bibinfo{year}{2017}\natexlab{b}.
\newblock \showarticletitle{Fairness constraints: Mechanisms for fair classification}. In \bibinfo{booktitle}{\emph{Artificial Intelligence and Statistics}}. PMLR, \bibinfo{pages}{962--970}.
\newblock


\bibitem[Zemel et~al\mbox{.}(2013)]%
        {pmlr-v28-zemel13}
\bibfield{author}{\bibinfo{person}{Rich Zemel}, \bibinfo{person}{Yu Wu}, \bibinfo{person}{Kevin Swersky}, \bibinfo{person}{Toni Pitassi}, {and} \bibinfo{person}{Cynthia Dwork}.} \bibinfo{year}{2013}\natexlab{}.
\newblock \showarticletitle{Learning Fair Representations}. In \bibinfo{booktitle}{\emph{Proceedings of the 30th International Conference on Machine Learning}} \emph{(\bibinfo{series}{Proceedings of Machine Learning Research}, Vol.~\bibinfo{volume}{28})}. \bibinfo{publisher}{PMLR}, \bibinfo{pages}{325--333}.
\newblock
\urldef\tempurl%
\url{https://proceedings.mlr.press/v28/zemel13.html}
\showURL{%
\tempurl}


\bibitem[Zeng et~al\mbox{.}(2022)]%
        {https://doi.org/10.48550/arxiv.2202.09724}
\bibfield{author}{\bibinfo{person}{Xianli Zeng}, \bibinfo{person}{Edgar Dobriban}, {and} \bibinfo{person}{Guang Cheng}.} \bibinfo{year}{2022}\natexlab{}.
\newblock \bibinfo{title}{Bayes-Optimal Classifiers under Group Fairness}.
\newblock
\newblock
\urldef\tempurl%
\url{https://doi.org/10.48550/ARXIV.2202.09724}
\showDOI{\tempurl}


\bibitem[Zhao and Gordon(2019)]%
        {zhao2019inherent}
\bibfield{author}{\bibinfo{person}{Han Zhao} {and} \bibinfo{person}{Geoffrey~J. Gordon}.} \bibinfo{year}{2019}\natexlab{}.
\newblock \showarticletitle{Inherent Tradeoffs in Learning Fair Representations}. In \bibinfo{booktitle}{\emph{NeurIPS}}. \bibinfo{pages}{15649--15659}.
\newblock


\end{thebibliography}
\end{document}